\documentclass[a4paper,12pt]{article} 

\usepackage{times}
\usepackage{soul}
\usepackage{url}
\usepackage[hidelinks]{hyperref}
\usepackage[utf8]{inputenc}
\usepackage[small]{caption}
\usepackage{graphicx}
\usepackage{amsthm}
\usepackage{booktabs}
\usepackage{algorithm}
\usepackage{algorithmic}
\usepackage{subfig}
\usepackage{microtype}
\usepackage{xspace}
\usepackage{amsmath}

\usepackage[fixamsmath,disallowspaces]{mathtools}

\usepackage{microtype}

\usepackage{macros}
\usepackage{comment}

\newtheorem{theorem}{Theorem}

\newtheorem{example}[theorem]{Example}
\newtheorem{lemma}[theorem]{Lemma}
\newtheorem{claim}[theorem]{Claim}
\newtheorem{definition}[theorem]{Definition}
\newtheorem{remark}[theorem]{Remark}

\usepackage{enumerate}
\usepackage{xcolor}

\usepackage{tikz}
\usetikzlibrary{arrows}
\usepackage{todonotes}

\usepackage{array}
\newcolumntype{H}{>{\setbox0=\hbox\bgroup}c<{\egroup}@{}}

\DeclareMathOperator{\fimp}{\mathbb{S}}

\DeclareMathOperator{\FA}{\mathcal{F}}

\newcommand{\SB}{\{}%
\newcommand{\SM}{\mid}%
\newcommand{\SE}{\}}%

\newcommand{\complexityClassFont}[1]{\ensuremath{\mathrm{#1}}}
\newcommand{\numberP}{\complexityClassFont{\#\Ptime}\xspace}
\newcommand{\numberDotP}{\complexityClassFont{\#{\cdot}\Ptime}\xspace}
\newcommand{\numberDotCoNP}{\complexityClassFont{\#\cdot\co\NP}\xspace}
\newcommand{\PH}{\text{\complexityClassFont{PH}}\xspace}
\newcommand{\PS}{\ensuremath{\Ptime^{\numberDotP}}\xspace}

\usepackage{thmtools}
\usepackage{thm-restate}
\usepackage{cleveref}

\usepackage{apptools}
\usepackage{xpatch}
\makeatletter
\xpatchcmd{\thmt@restatable}
{\csname #2\@xa\endcsname\ifx\@nx#1\@nx\else[{#1}]\fi}
{\IfAppendix{\csname #2\@xa\endcsname}{\csname #2\@xa\endcsname\ifx\@nx#1\@nx\else[{#1}]\fi}}
{}{} 
\makeatother
\usepackage{authblk}

\begin{document} 

\title{Facets in Argumentation: A Formal Approach to Argument Significance}

\author[1]{Johannes K.\ Fichte \thanks{\texttt{johannes.klaus.fichte@liu.se}}}
\author[2]{Nicolas Fröhlich \thanks{\texttt{nicolas.froehlich@thi.uni-hannover.de}}}
\author[3,4]{Markus Hecher \thanks{\texttt{hecher@mit.edu}}}
\author[1]{Victor Lagerkvist \thanks{\texttt{victor.lagerkvist@liu.se}}}
\author[5]{Yasir Mahmood \thanks{\texttt{yasir.mahmood@uni-paderborn.de}}}
\author[2]{Arne Meier \thanks{\texttt{meier@thi.uni-hannover.de}}}
\author[1]{Jonathan Persson \thanks{\texttt{jonpe481@student.liu.se}}}

\affil[1]{Department of Computer and Information Science, Link\"oping University, Sweden}
\affil[2]{Leibniz Universit\"at Hannover, Germany}
\affil[3]{Univ. Artois, CNRS, France}
\affil[4]{CSAIL, Massachusetts Institute of Technology, USA}
\affil[5]{DICE group, Paderborn University, Germany}

\date{\vspace{-5ex}}  

\maketitle

\begin{abstract}
	Argumentation is a central subarea of Artificial Intelligence (AI) for modeling and reasoning about arguments. 
	The semantics of abstract argumentation frameworks (AFs) is given by sets of arguments (extensions) and conditions on the relationship between them, such as stable or admissible.
	Today's solvers implement tasks such as finding extensions, deciding credulous or skeptical acceptance, counting, or enumerating extensions.
	While these tasks are well charted, the area between decision, counting/enumeration and fine-grained reasoning requires expensive reasoning so far.
	We introduce a novel concept (facets) for reasoning between decision and enumeration. 
	Facets are arguments that belong to some  extensions (credulous) but not to all extensions (skeptical).
	They are most natural when a user aims to navigate, filter, or comprehend the significance of specific arguments, according to their needs.
	We study the complexity and show that tasks involving facets are much easier than counting extensions.
	Finally, we provide an implementation, and conduct experiments to demonstrate feasibility.
\end{abstract}

\section{Introduction}
Abstract argumentation~\cite{Dung95a,Bench-CaponDunne07} is a
formalism for modeling and evaluating arguments and its reasoning problems has many applications in artificial intelligence
(AI)~{\cite{AmgoudPrade09a,RagoCocarascuToni18a}}.
The semantics is based on sets of arguments that satisfy certain
conditions regarding the relationship among them, such as being stable
or admissible~\cite{Dung95a}. Such sets of arguments are then called \emph{extensions} of a
framework and {various practical
solvers for decision and reasoning tasks~\cite{EglyGagglWoltran08,NiskanenJarvisalo20,ThimmCeruttiVallati21,Alviano21}
compete biannually in the ICCMA competition~\cite{ThimmEtAl24}}.

Qualitative reasoning problems such as finding an extension or
deciding credulous or skeptical acceptance are reasonably fast to
compute~\cite{Dvorak12a} but have
limitations.
{Namely, these two reasoning modes represent extremes on the reasoning spectrum, as they provide no insight into preferences among arguments for further analysis}.
As a result, enumeration, counting, and fine-grained quantitative
reasoning modes have been studied and computationally
classified~\cite{FichteHecherMahmood23,FichteHecherMeier24} enabling
probabilistic reasoning over arguments.
While enumeration is well suited when the total number of extensions
is small, some argumentation semantics easily result in a vast number
of extensions.
However, users might still want to investigate the space of possible
extensions in more detail. 
Possible examples are restricting or diversifying extensions, identifying resilient arguments or sanity checks, evaluating outcomes in argumentation frameworks generated by LLMs, or gaining insights into specific frameworks through explanations.

In all such scenarios evaluating the \emph{significance} for individual arguments in a framework is central.
In existing proposals, computing the significance for arguments relied on quantitative measures over extensions containing certain arguments or supporting particular claims. 
These notions rely on counting all extensions containing a particular argument (or claim), which is computationally expensive~\cite{FichteHecherMeier24}.
Example~\ref{ex:limits} 
illustrates difficulties when comparing significance of certain
arguments in the overall world of extensions.

\begin{example}\label{ex:limits}
	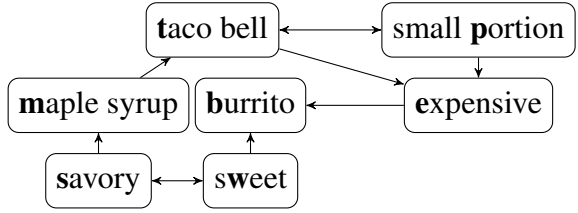
\begin{figure}
	  \centering
	  \begin{tikzpicture}[
		y=1cm,
		baseline=(current bounding box.center),
		every node/.style={rounded corners, draw, inner ysep = 3pt},
		execute at end node={\strut}
		]
		\node at (2, 0) (sweet) {s\textbf{w}eet}; 
		\node at (0, 0) (savory) {\textbf{s}avory};
		\node at (2, 1) (burrito) {\textbf{b}urrito};
		\node at (0, 1) (maple syrup) {\textbf{m}aple syrup};
		\node at (1.5, 2) (taco bell) {\textbf{t}aco bell};
		\node at (5, 1) (expensive) {\textbf{e}xpensive};
		\node at (5, 2) (small portion) {small \textbf{p}ortion};
	
		\draw[stealth'-stealth'] (sweet) -- (savory);
		\draw[-stealth'] (savory) -- (maple syrup);
		\draw[-stealth'] (sweet) -- (burrito);
		\draw[-stealth'] (maple syrup) -- (taco bell);
		\draw[-stealth'] (taco bell) -- (expensive);
		\draw[stealth'-stealth'] (small portion) -- (taco bell);
		\draw[-stealth'] (small portion) -- (expensive);
		\draw[-stealth'] (expensive) -- (burrito);
	  \end{tikzpicture}
	  \caption{An example argumentation framework.}
	  \label{fig:running}
	\end{figure}
  Consider the argumentation framework $F$, depicted in Figure~\ref{fig:running}, illustrating the choice between a sweet or savory breakfast, that is between maple syrup and burrito.
  Intuitively, if one prefers a \emph{savory} flavor, they would not choose \emph{maple syrup}; likewise, \emph{sweet} attacks \emph{burrito}.
  One does not go to \emph{taco bell} expecting \emph{maple syrup}, nor are \emph{small portions} typical at \emph{taco bell}.
  While it is possible to make \emph{burritos} at home, doing so requires buying \emph{expensive} ingredients.
  Making a \emph{small portion} or going to \emph{taco bell} avoids this.
 
  The stable extensions of $F$ are: $\{w, m, p\}$, $\{s, b, p\}$ and $\{s, b, t\}$. 
  Now it is not immediate to compare the significance of accepting/rejecting certain arguments to each other.
\end{example}

In this paper, we propose 
a combination of credulously and skeptically accepted arguments, which
ask whether a given argument belongs to some extension (credulous) but not
all extensions (skeptical).  We call arguments that are credulously
but not skeptically accepted \emph{facets}.
Facets quantify the uncertainty of arguments in extensions, providing a measure of their indeterminacy within the framework.
{They can be utilized to evaluate the significance of 
specific arguments.}

Example~\ref{ex:running} provides a brief intuition.

\begin{example}\label{ex:running}
  We return to the argumentation framework $F$ from Example~\ref{ex:limits}.
  Six of the seven arguments are facets under stable semantics, with only $e$ being a non-facet, indicating a substantial degree of uncertainty.

  Assume we aim to compare the relative significance of an argument.
  Consider the extensions of $F$ {rejecting} (not containing) the argument ``s\textbf{w}eet''. 
  There are two such stable extensions $\{s, b ,p\}$ and $\{s, b, t\}$.
  Therefore, rejecting the argument $w$ leaves us with two facets $p$ and $t$.
  In contrast, consider the extensions \emph{accepting} (containing) the argument ``s\textbf{w}eet''. 
  This results in one stable extension $\{w, m, p\}$ and hence no facets.
  Accepting the argument $w$ eliminates any uncertainty, whereas rejecting $w$ does not.
  Consequently, we consider accepting ``s\textbf{w}eet''
  to be more significant than rejecting ``s\textbf{w}eet''. 
\end{example}

While the computational complexity of credulous and skeptical
reasoning is well studied~\cite{Dvorak12a}, we ask for the concrete
complexity of facets and whether counting facets provides a theoretical benefit
over projected counting and projected enumerating extensions.

\paragraph*{Contributions.} In more details, we establish the following.
\begin{enumerate}
\item We introduce facets to abstract argumentation as a reasoning
  tool for significance and filtering extensions in a
  directed way.  By this, we fill a gap in the literature between
  quantitative and qualitative reasoning.
\item We present a comprehensive complexity analysis for various
  qualitative and quantitative problems involving facets.
  Table~\ref{tab:results} provides an overview on our results. 
\item Finally, we present experiments that demonstrate the feasibility
  of our framework.  We
  evaluate our implementation on instances of the ICCMA competition.
          
\end{enumerate}

\begin{table}
  \centering
  \begin{tabular}{lccc}%
    \toprule%
	Problems/$\sem$ & $\sigma_1 $ & $\sigma_2$ & $\sigma_3 $                       \\
	\midrule 
	$\isfacet_\sigma$ 
                        & $\Ptime^{\footnotesize{\text{ R}\ref{rem:conf-isfacet}}}$ 
                        & $\NP^{\footnotesize{\text{ T}\ref{thm:all-isfacet}}}$ 
                        & ${\SigmaP}^{\footnotesize{\text{ T}\ref{thm:all-isfacet}/\ref{thm:pref-isfacet}}}$

                                                                                       \\
	$\atleastkfacets_\sigma$ 
                        & $\Ptime^{\footnotesize{\text{ T}\ref{thm:conf-boundk}}}$
                        & $\NP^{\footnotesize{\text{ T}\ref{thm:adm-atleastk}}}$
                        & ${\SigmaP}^{\footnotesize{\text{ T}\ref{thm:pref-atleastk}}}$

                                                                                       \\
	$\atmostkfacets_\sigma$ 
                        & $\Ptime^{\footnotesize{\text{ T}\ref{thm:conf-boundk}}}$
                        & $\co\NP^{\footnotesize{\text{ T}\ref{thm:all-atmostk}}}$
                        & ${\PiP}^{\footnotesize{\text{ T}\ref{thm:all-atmostk}}}$

                                                                                       \\
	$\exactkfacets_\sigma$ 
                        & $\Ptime^{\footnotesize{\text{ T}\ref{thm:conf-exactk}}}$
                        & ${\DP}^{\footnotesize{\text{ T}\ref{thm:adm-exactk}}}$ 
                        & $\in{\DP_2}^{\footnotesize{\text{ T}\ref{thm:pref-exactk}}}$ \\

     \midrule
	$\fimp_\sem[F,\ell]$ 
		& ``$\in \Ptime$''$^{\footnotesize{\text{ T}\ref{thm:conf-exactk}}}$
		& ``$\in {\DeltaP 2}$''$^{\footnotesize{\text{ T}\ref{thm:sig}}}$
		& ``$\in {\DeltaP 3}$''$^{\footnotesize{\text{ T}\ref{thm:sig}}}$ \\
	\bottomrule
  \end{tabular}
  \caption{%
    Overview of our complexity results for the semantics~$\sigma_1\in \{\cnf, \nai\}$,
    $\sigma_2\in \{\adm, \stab, \comp\}$, and
    $\sigma_3\in \{\pref, \semi,$ $\stag\}$.
    All results depict completeness except for $\Ptime$-cases or when stated o/w.
    $\isfacet_\sigma$  asks whether a given argument is a facet;
    $\cntopfacets{\cdot}_\sigma$ asks whether there are at least
    ($\geq k$), at most ($\leq k$), or exactly (\mbox{$=k$}) facets.
  	$\fimp_\sem[F,\ell]$ asks what the significance of approving ($\ell = a$) or disapproving ($\ell = \bar a$) of a $\sem$-facet $a$ in AF $F$ is.
    Superscripts behind the complexity classes refer to Remark (R) and Theorem (T),
    with the proof. 
    ``$\in {\DeltaP i}$'' slightly abuses notation
    meaning that it can be computed by
    a deterministic polynomial-time Turing machine with access to a
    $\SigmaPtime{i-1}$ oracle.
  }
  \label{tab:results} %
\end{table}

\paragraph*{Related Work.}
The computational complexity in abstract argumentation is well
understood for decision
problems~\cite{DunneBench-Capon02a,DvorakWoltran10,Dvorak12a},
parameterized complexity involving treewidth~\cite{FichteH0M21}, for (projected) counting~\cite{FichteHecherMeier24} as well as for
fine-grained reasoning based on counting~\cite{FichteHecherMahmood23}.
{The decision complexity} ranges from $\Ptime$ to $\SigmaP$ for
credulous reasoning, and from $\Ptime$ to $\PiP$ for skeptical reasoning.
The complexity of counting extensions ranges between $\numberP$ and $\numberDotCoNP$ depending on
the semantics. In theory, we know that
$\PH \subseteq \PS$~\cite{Toda91} where
$\bigcup_{k \in \mathbb{N}}\DeltaP{k} = \PH$ and
$\NP \subseteq \DeltaP2 = \Ptime^{\NP}$~\cite{Stockmeyer76}.
This renders counting extensions theoretically significantly harder.
Approximate counting is in fact
easier,~i.e.,~$\text{approx}\hy\numberDotP \subseteq \text{BPP}^{\NP}
\subseteq \SigmaPtime3$~\cite{Lautemann1983,Sipser1983,Stockmeyer1983},
but turns out to be still harder than counting facets.
Facets were initially proposed for answer-set programming (ASP)
by~\cite{AlrabbaaRudolphSchweizer18} as a tool to navigate large
solution spaces.
Their computational complexity has been
systematically classified~\cite{RusovacHecherGebser24}.
ASP is a popular framework and problem solving paradigm to model and
solve hard combinatorial problems in form of a logic program that
expresses constraints~\cite{GebserKaufmannSchaub12a}.
Plausibility reasoning has been developed for ASP based on
full counting~\cite{FichteHecherNadeem22}.
\cite{DachseltGagglKrotzsch22} developed a tool to navigate
argumentation frameworks using ASP-facets. 
Note that ASP-navigation is
based on forbidding or enforcing atoms in a program using integrity
constraints.
In contrast, argumentation facets enable approving or disapproving
arguments, while not necessarily removing the extensions entirely
leading to a natural notion of significance of an argument (see
Section~\ref{sec:significance}).
Finally, we remark that our complexity analysis (Table~\ref{tab:results}) indicates a computational gain for reasoning with facets compared to separately asking credulous/skeptical reasoning in each case (see e.g.,~\cite{flap/DvorakD17}).
Facets have recently also been applied to planning~\cite{SpeckHecherGnad25}.

\section{Preliminaries}

We assume familiarity with computational complexity~\cite{DBLP:books/daglib/0092426}, graph theory~\cite{DBLP:books/sp/BondyM08}, and Boolean  logic~\cite{DBLP:series/faia/336}.

\paragraph{Complexity Classes}
We use standard notation for basic complexity classes and for example write $\Ptime$ ($\NP$) for the class of decision problems solvable in (non-deterministic) polynomial time. Additionally, we let $\co \NP$ be the class of decision problems whose complement is in $\NP$, and let $\DP$ be the class of decision problems representable as the intersection of a problem in $\NP$ and a problem in $\co \NP$.
On top, we use more classes from the polynomial hierarchy~\cite{StockmeyerMeyer73,Stockmeyer76,Wrathall76},
$\DeltaP0 \coloneqq \PiPtime0 \coloneqq
\SigmaPtime0 \coloneqq \Ptime$ and $\DeltaP{i} \coloneqq
\Ptime^{\SigmaPtime{i-1}}$, $\SigmaPtime{i} \coloneqq
\NP^{\SigmaPtime{i-1}}$, and $\PiPtime{i} \coloneqq
\co\NP^{\SigmaPtime{i-1}}$ for $i>0$ where $C^{D}$ is the class~$C$ of
decision problems augmented by an oracle for some complete problem in
class $D$.
Recall that $\PH \coloneqq \bigcup_{i \in \mathbb{N}}
\DeltaP{i}$ \cite{Stockmeyer76}.
The canonical NP-complete problem is the Boolean {\em satisfiability} problem for formulas in {\em conjunctive normal form} (CNF), i.e., given $\varphi\dfn\bigwedge_{i=1}^m C_i$ where each $C_i$ is a clause, decide whether $\varphi$ admits at least one satisfying assignment. For $\co \NP$ the corresponding problem is simply to check {\em unsatisfiability}, and for $\DP$ to check whether $\varphi$ is satisfiable and $\psi$ unsatisfiable for a given pair of formulas $(\varphi, \psi)$ (the SAT-UNSAT problem).
The complexity class $\DP_k$ is defined as
$\DP_k\eqdef \SB L_1 \cap L_2 \SM L_1\in \SigmaPtime{k}, L_2 \in
\PiPtime{k}\SE$, $\DP{=}\DP_1$~\cite{LohreyRosowski23}.
For $\PiP$ we may e.g.\ consider the evaluation problem for a {\em quantified Boolean formula} of the form $\forall X \exists Y . \varphi$ where $X$ and $Y$ are two disjoint sets of variables and $\varphi$ a formula in CNF over $X$ and $Y$. For $\SigmaP$ the problem is instead to check that $\forall X \exists Y . \varphi$ is false.

\paragraph{Abstract Argumentation}
We use Dung's argumentation framework~(\cite{Dung95a}) and consider only non-empty and finite sets of arguments~$A$.
An \emph{(argumentation) framework~(AF)} is a directed graph~$F=(A, R)$, where $A$ is a set of arguments and $R \subseteq A\times A$, consisting of pairs of arguments
representing direct attacks between them.
\longversion{Let the \emph{direct attack relationship} between two sets~$E,E'\subseteq A$ of arguments be defined by relations $\rightarrowtail_R$ and $\leftarrowtail_R$ as follows:
$E$ \emph{directly attacks}, i.e.,  $E\rightarrowtail_R E'\eqdef \{a\in E \mid (\{a\}\times E')\cap R \neq \emptyset\}$, 
and $E\leftarrowtail_R E'\eqdef \{a\in E \mid (E' \times \{a\}) \cap R\neq \emptyset\}$. }%
An argument~$a \in E$, is called \emph{defended by $E$ in $F$} if for every $(a', a) \in R$, there exists $a'' \in E$ such that $(a'', a') \in R$.  
The family~$\adef_F(E)$ is defined by $\adef_F(E) \eqdef\{\; a \mid a \in A, a \text{ is defended by $E$ in $F$}  \;\}$.
In abstract argumentation, one strives for computing so-called \emph{extensions}, which are subsets~$E \subseteq A$ of the arguments that have certain properties.
The set~$E$ of arguments is called \emph{conflict-free in~$E$} if $(E\times E) \cap R = \emptyset$; $E$ is \emph{admissible in $F$} if 
(1) $E$ is \emph{conflict-free in $F$}, and
(2) every $a \in E$ is \emph{defended by $E$ in $F$}.
Let $E^+_R\eqdef E\cup\{\, a\mid (b,a)\in R, b \in E\, \}$ and $E$ be conflict-free. 
Then, $E$ is (1) \emph{naive} in $F$ if no $E' \supset E$ exists that is {conflict-free in $F$}, and
(2) \emph{stage in $F$} if there is no conflict-free set~$E'\subseteq A$ in~$F$ with~$E^+_R\subsetneq (E')^+_R$.
An admissible set $E$ is 
(1) \emph{complete in~$F$} if $\adef_F(E) = E$;
(2) \emph{preferred in~$F$}, if no $E' \supset E$ exists that is \emph{admissible in $F$};
(3) \emph{semi-stable in $F$} if no admissible set $E' \subseteq A$ in~$F$ with~$E^+_R\subsetneq (E')^+_R$ exists; and 
(4) \emph{stable in~$F$} if every $a \in A \setminus E$ is \emph{attacked} by some $a' \in E$.
For a semantics~$\sem \in \{\cnf, \nai, \adm, \comp, \stab, \pref,\semi,\stag\}$, we write $\sem(F)$ for the set of \emph{all extensions} of semantics~$\sem$ in $F$.
Let~$F=(A,R)$ be an AF.

Then, the problem $\Exist\sem$ asks if $\sem(F)\neq\allowbreak\emptyset$. 
The problems $\cred{\sem}$ and $\skep{\sem}$ question for 
$a{\,\in\,}A$, whether~$a$ is in some $E\in \sem(F)$ (``\emph{credulously} accepted'') or every~$E\in\sem(F)$  (``\emph{skeptically} accepted''), respectively. 
We let $\Cred\sigma$ (resp., $\Skep\sigma$) denote the set of all credulously (skeptically) accepted arguments under semantics $\sigma$.

\section{Facet Reasoning} 

Central reasoning problems in argumentation include deciding whether an argument is credulously (or skeptically) accepted. 
In the following, we will see how the problem of deciding facets can be seen as a generalization of these two modes. 

We begin by defining reasoning problems pertaining to \emph{facets} in argumentation.
Intuitively, a $\sem$-facet is an argument which is accepted in some, but not all $\sigma$-extensions for the considered semantics $\sigma$.
Formally, given a semantics $\sigma$, then an argument $a$ is a {\em $\sigma$-facet} if $a\in \Cred\sigma\setminus \Skep\sigma$.
Given an AF $F$ and semantics $\sem$, then $\Facet\sem(F)$ denotes the set of all $\sem$-facets in $F$
In this work, we consider the following reasoning problems parameterized by a semantics~$\sem$.
\begin{itemize}
  \item \problemdef{$\isfacet_\sem$}{an AF $F=(A,R)$ and an argument $a\in A$}{is $a$ a $\sem$-facet in $F$}
  \item 
    The problems $\exactkfacets_\sem,\atleastkfacets_\sem$ and $\atmostkfacets_\sem$ has an integer $k$ as an additional input, and ask whether an input $F=(A,R)$ has exactly, at least, or at most $k$ $\sem$-facets, respectively.
    \end{itemize}

We continue by analyzing the complexity of these problems, beginning with $\isfacet_\sem$ in Section~\ref{sec:isfacet}, $\atleastkfacets_\sem$ and $\atmostkfacets_\sem$ in Section~\ref{sec:atleast_atmost}, and complete the study with $\exactkfacets_\sem$ in Section~\ref{sec:exact}.

\subsection{Complexity of Deciding Facets} 
\label{sec:isfacet}

Regarding the complexity classification, for conflict-free and naive semantics, the problem $\isfacet$ is rather straightforward to classify. 
To see this,
	for $\cnf$, each argument not attacking itself is a facet (provided there are at least two such arguments in the AF), and
	for $\nai$, one additionally has to remove each argument not in conflict with any other argument since it can not be a facet. 
\begin{remark}\label{rem:conf-isfacet}
	$\isfacet_\sigma$ is in $\Ptime$ for $\sigma\in \{\cnf,\nai\}$.
\end{remark}

For the remaining semantics, 
we get hardness by observing that $\isfacet_\sigma$ is as hard as the credulous reasoning ($\cred\sigma$) for each considered semantics $\sem$.

\begin{lemma}\label{lem:cred-isfacet}
	Let $\sigma$ be any semantics. Then $\cred\sigma \leq^{\Ptime}_m \isfacet_{\sigma}$.
\end{lemma}

\begin{proof}
	We provide a polynomial time many-one reduction from $\cred\sigma$ to $\isfacet_\sem$ for each semantics $\sigma$ as follows.
	Let $F=(A,R)$ be an AF, and $a\in A$ be an argument. 
	Our reduction yields an AF $F'$ where we duplicate the argument $a$ which has all the incoming and outgoing attacks similar to $a\in A$. Precisely, $F'=(A',R')$ is 
        as follows:
	\begin{itemize}
		\item $A' \dfn A\cup\{a'\}$ for a fresh $a'\not\in A$,
		\item $R' \dfn R \cup\{(a,a'),(a',a)\}\cup \{(a',x)\mid (a,x)\in R\}\cup \{(x,a')\mid (x,a)\in R\}$.
	\end{itemize}
	Then, for any semantics $\sigma$, the argument $a$ is credulously accepted under $\sem$ in $F$ iff $a$ is a facet under $\sigma$ in $F'$.
	Indeed, let $a$ be credulously accepted, then there is a $\sigma$-extension $E\subseteq A$ such that $a\in E$. 
	Since $a$ defends itself against $a'$ in $F'$, we have that $a$ is also credulously accepted in $F'$ as the other attacks remain the same.
	Finally, $a$ can not be in all $\sigma$-extensions $E$ in $F'$, as $E\setminus \{a\}\cup \{a'\}$ is also a $\sigma$-extension.
	Therefore the claim follows.
\end{proof}

We specifically obtain the following characterization.

\begin{theorem}\label{thm:all-isfacet}
	$\isfacet_\sigma$ is $\NP$-complete for $\sigma\in \{\adm,\comp, \stab\}$ and $\SigmaP$-complete for $\sigma\in\{\semi,\stag\}$.
\end{theorem}

\begin{proof}
	The hardness in each case follows due to Lemma~\ref{lem:cred-isfacet} and the complexity for credulous reasoning under corresponding semantics.
	
	The membership follows, since one can guess two $\sigma$-extensions for an input $F$, one containing the argument in the question, and another without it. 
	The verification (of $\sigma$-extensions) requires (1) polynomial time for $\sigma\in\{\adm,\comp,\stab\}$, and (2) $\co\NP$-oracle for $\sigma\in \{\semi,\stag\}$. 
	This establishes the 
        membership results.
\end{proof}

The following observation is necessary to establish the hardness proof in Theorem~\ref{thm:pref-isfacet}.
	Let $\Phi = \forall X \exists Y . \varphi$ be a QBF instance, where $\varphi \dfn \bigwedge_{i=1}^nC_i$ is a CNF.
	If $\varphi$ is not satisfiable, the formula $\Phi$ can not be true.
	Whereas, the problem to check  whether $\varphi$ is satisfiable, is $\NP$-complete.
	Given $\Phi$, one can construct a new formula $\Phi'= \forall X' \exists Y .\varphi'$ such that $\varphi'$ is satisfiable and $\Phi$ is true iff $\Phi'$ is true.
	To this aim, we let  $X' = X\cup \{z\}$ for a fresh variable $z\not\in X\cup Y$. 
	Then $\varphi' \dfn \bigwedge_{C\in\varphi}(\neg z\lor C)$.
	Notice that $\Phi[z\mapsto 0]$ is trivially true whereas $\Phi[z\mapsto 1]$ is true iff $\Phi$ is true.
	Consequently, we have the following observation.

	\begin{remark}
		Given a QBF instance $\Phi = \forall X \exists Y . \varphi$, where $\varphi$ is a CNF.
		One can assume w.l.o.g. that $\varphi$ is satisfiable.
	\end{remark}
	Next, we  establish that $\isfacet_\pref$ is $\SigmaP$-complete.
	Notice that this case is not covered by Theorem~\ref{thm:all-isfacet} as the credulous reasoning for preferred semantics ($\cred\pref$) is \emph{only} $\NP$-complete.

\begin{theorem}\label{thm:pref-isfacet}
	$\isfacet_\pref$ is $\SigmaP$-complete.
\end{theorem}

\begin{proof}
	The membership follows since one can guess two preferred extensions for an input $F$, one containing the argument in the question, and another without it. 
	The verification of $\pref$-extensions requires an $\NP$-oracle. 
	This establishes the mentioned membership results.
	
	For hardness, we utilize the following reduction proving $\PiP$-hardness of skeptical acceptance with preferred semantics~\cite[Reduction 3.7]{flap/DvorakD17}.

	Given a QBF $\Phi = \forall Y \exists Z .\varphi$ where $\varphi \dfn  \bigwedge_{i=1}^mC_i$ is a CNF-formula with clauses $C_i$ over variables $X=Y\cup Z$.
	We construct an AF $F_\Phi = (A,R)$, where 
	$A=\{\varphi,\bar\varphi\}\cup \{C_1,\dots,C_m\}\cup X \cup \bar X$.
	The relation $R$ includes the following attacks:
	\begin{align*}
		&\{\,(C_i,\varphi)\mid 1\leq i\leq m\,\}\cup\\
		&\{\,(x,C_i)\mid x\in C_i\,\}\cup\{\,(\bar x,C_i)\mid \bar x\in C_i\,\}\cup\\
		&\{\,(x,\bar x),(\bar x,x)\mid x\in\var(\varphi)\,\} \cup \\ 
		& \{(\varphi,\bar\varphi), (\bar\varphi,\varphi)\} \cup \{(\bar\varphi,z),(\bar\varphi,\bar z)\mid z\in Z\}.
	\end{align*}
	Then, it holds that there is a preferred extension in $F_\Phi$ not containing the argument $\varphi$ iff the formula $\Phi$ is false.
	Furthermore, we have that $\varphi$ is satisfiable.
	Therefore, there exists a preferred extension $S$ containing $\varphi$. 
	Namely, $S$ corresponds to a satisfying assignment $\theta$ for $\varphi$ as $S=\{\varphi, \ell \mid \ell \in X\cup \bar X, \theta(\ell)=1\}$.
	As a result, $\varphi$ is a $\pref$-facet iff there is a preferred extension $S'$ with $\varphi\not\in S'$ iff the formula $\Phi$ is false.
	This results in $\SigmaP$-hardness.
\end{proof}

\subsection{Atleast/Atmost $k$ Facets Complexity}
\label{sec:atleast_atmost}
	We begin by proving that for conflict-free and naive semantics, one can count all the facets in polynomial time.

\begin{theorem}\label{thm:conf-boundk}
	$\atleastkfacets_\sigma$ and $\atmostkfacets_\sigma$ are both in $\Ptime$ for $\sigma\in \{\cnf,\nai\}$.
\end{theorem}
\begin{proof}
	For $\cnf$, the AF $F$ has at least $k$ facets if $F$ contains at least $k$ non self-conflicting arguments (without self-attacks).
	For $\nai$, one additionally has to remove arguments ($N$) not in conflict with any other argument since those can not be a facet.
        Thus, $F$ has at least $k$ $\nai$-facets if $F\setminus N$ contains at least $k$ non self-conflicting arguments, where $N$ includes those arguments not participating in any attack.
\end{proof}

The following reduction is essential for proving Lemma~\ref{lem:sat-kfacets} which we utilize later in achieving certain lower bounds.

\begin{definition}[\cite{flap/DvorakD17}]\label{def:translation}
	Let $\varphi\dfn\bigwedge_{i=1}^m C_i$, be a CNF-formula where each $C_i$ is a clause.
	Consider the AF $F_\varphi=(A,R)$ constructed as follows:
	\begin{align*}
		A\coloneqq &\{\varphi, C_1,\dots,C_m\}\cup\{\,x,\bar x\mid x\in\var(\varphi)\} \\
		R \coloneqq &\{\,(C_i,\varphi)\mid i\leq m\,\}\cup \{\,(x,\bar x),(\bar x,x)\mid x\in\var(\varphi)\,\}\\
		&\cup\{\,(x,C_i)\mid x\in C_i\,\}\cup\{\,(\bar x,C_i)\mid \bar x\in C_i\,\}.
	\end{align*}
	We call $F_\varphi$ the argumentation framework of  $\varphi$ {generated} via the standard translation.
\end{definition}

 It is known that $\varphi$ is satisfiable iff the argument $\varphi$ is credulously accepted in $F_\varphi$ under semantics $\sigma\in\{\adm,\comp,\stab\}$.
 We next prove the following intermediate lemma.
 Essentially, the standard translation allows us to characterize exactly the number of facets in $F_\varphi$ based on whether the formula $\varphi$ is satisfiable or not.

 \begin{restatable}{lemma}{satkfacets}\label{lem:sat-kfacets}
	Let $\varphi$ be a CNF-formula involving $m$ clauses and $n$ variables.
	Moreover, let $F_\varphi$ be the AF of $\varphi$ as depicted in Definition~\ref{def:translation} and let $k = 2n+m+1$. 
	Then the following statements are true for every $\sigma\in\{\adm,\comp,\stab\}$.
	\begin{enumerate}
		\item $\varphi$ is satisfiable iff $F_\varphi$ admits exactly $k$ $\sigma$-facets.
		\item $\varphi$ is not  satisfiable iff $F_\varphi$ admits exactly $k-1$ $\sigma$-facets.
	\end{enumerate}
      \end{restatable}

\begin{proof}
	First, we assume, w.l.o.g., that for all $1\leq i \leq m$, $C_i\not\equiv \top$.
	As a result, for every $C_i\in\varphi$ there exists an assignment $\theta_i$ such that $\theta_i\not\models C_i$.
	Further, it follows that such $\theta_i$ do not satisfy $\varphi$, and hence $\varphi$ is not tautological.
	Next, let $\theta$ be an arbitrary assignment over $\var(\varphi)$. Then, for each variable $x$, we have that either $x\in\theta$ or $\bar x\in \theta$ and not both.
	Thus, $\theta$ (seen as set of arguments) is conflict-free.
	By slightly abusing the notation, for a literal argument $\ell\in \theta$, we write $\bar\ell = \bar x$ if $\ell=x$ and $\bar\ell = x$ if $\ell = \bar x$.

	We prove the following intermediate claims.

	\begin{claim}\label{claim:clause-facets}
		Every clause argument $C_i\in A$ is a $\sem$-facet in $F_\varphi$ for each $\sigma\in\{\adm,\comp,\stab\}$.
	\end{claim}
	\noindent\textbf{Proof of Claim.}
	We prove the following two statements. That is, every clause argument $C_i\in A$ is contained in some $\sem$-extension, but not all $\sem$-extensions.

	\textbf{(C1) Every $C_i$ is credulously accepted under $\sem$.}
	This holds since  $\{C_i\}\cup \{\bar\ell\mid \ell\in C_i\}$ is an admissible set of arguments.
	Therefore each clause $C_i$ is credulously accepted under admissible semantics.
	Furthermore, we can \emph{combine} multiple clauses together with their non-satisfying assignments to yield results for complete and stable semantics.
	To this aim, let $S_i=\{C_i\}\cup \{\bar\ell\mid \ell\in C_i\}$.
	Consider a clause $C_j$ such that $C_j\cap S_i =\emptyset$.
	This is the case iff $C_j$ does not contain any literal $\ell$ with $\bar\ell\in C_i$.
	Then, for $S_j= \{C_j\}\cup \{\bar\ell \mid \ell\in C_j\}$, the set $S_i\cup S_j$ is conflict-free.
	We repeat this process to get a maximal extension $S_m$ such that: each clause $C_k\in\varphi\setminus S_m$ contains a literal $\ell\in S_m$, and hence it is attacked by $S_m$.
	Now consider the literals $\ell$ such that $\ell,\bar\ell\not\in S_m$. 
	Recall that $\ell,\bar \ell\not \in C$ for any $C\in S_m$, thus both literals are conflict free with $S_m$, and we get two stable extensions by taking both literals in turn (i.e., $S_m\cup\{\ell\}, S_m\cup\{\bar\ell\}$).
	Since one can construct such an admissible extension for each clause, and those extensions can be merged together to get a stable (hence also complete) extension, we have that each $C_i$ is in some $\sigma$-extension for each $\sigma\in\{\adm,\comp,\stab\}$.

	\textbf{(C2) No $C_i$ is skeptically accepted under $\sem$.}
	Let $\theta$ be an arbitrary assignment such that $\theta\models C_i$.
	Recall that $\theta$ (seen as set of arguments) is conflict-free as either $x\in\theta$ or $\bar x\in \theta$ (and not both) for each variable $x$.
	We have the following 
        cases:\\
	(I.) $\theta\models \varphi$.
	Then $\theta\cup\{\varphi\}$ is a $\sem$-extension not containing the argument $C_i$. \\
	(II.) $\theta\not\models \varphi$.
	Then, the set $C= \{C_j\mid \theta\not\models C_j\}$ of clauses is non-empty and $C_i\not\in C$.
	Moreover, the set of arguments $\theta \cup C$ is conflict-free since no $C_j\in C$ is attacked by any literal $\ell\in \theta$ (i.e., there is no $\ell\in \theta \cap C_j$ for any $C_j\in C$).
	The set $\theta \cup C$ is also admissible as any $C_j\in C$ is defended against the attacker $\ell \in C_j$ by $\bar \ell \in \theta$ (recall that $\theta$ either contains $x$ or $\bar x$ for every variable $x$).

	We next observe that $\theta \cup C$ is also stable, as
	(i.) every literal $\ell\not\in\theta$ is attacked by $\bar\ell\in\theta$,
	(ii.) every clause not in $C$ is attacked by its satisfying literal in $\theta$
	(iii.) $\varphi$ is attacked by each clause in $C$, which is non-empty.

	We next prove that $\theta \cup C$ is also complete.
	That is, $\theta\cup C$ contains every argument it defends.
	Note that
	(i.) the case for literals is trivial,
	(ii.) let $C_j$ be a clause defended by $\theta\cup C$.
	Since arguments in $C$ do not attack any literal, and those literals are the only attackers of $C_j$, we have that $C_j$ is actually defended by $\theta$.
	In other words, for each $\ell \in C_j$ we have that $\bar\ell \in \theta$ since literals are only attacked by their opposites. But this results in $\theta\not\models C_j$, and hence $C_j\in C$.
	(iii.) $\varphi$ is not defended by $\theta \cup C$ as there are unsatisfied clauses in $C$.
	Thus $\theta\cup C$ contains every argument it defends, making it a complete extension.

	From (I.) and (II.), we conclude that for each $C_i\in A$, there is a $\sem$-extension for $\sem\in\{\adm,\comp,\stab\}$ that does not contain $C_i$.

	\begin{claim}\label{claim:literal-facets}
	Every literal argument $\ell\in A$ is a $\sem$-facet in $F_\varphi$ for each $\sigma\in\{\adm,\comp,\stab\}$.
	\end{claim}
	\noindent\textbf{Proof of Claim.}
	Let $\theta$ be an arbitrary assignment.
	Recall that $\theta$ is a conflict-free set of arguments in $F_\varphi$.
	It suffice to prove that every literal in $\theta$ is credulously accepted in $F_\varphi$ under semantics $\sem$.
	That every literal in $\theta$ is a $\sem$-facet, follows by considering the \emph{dual} assignment $\bar\theta=\{\bar \ell \mid \ell\in\theta\}$ for $\theta$.
	First, we distinguish the following two cases.

	(I.) $\theta\models \varphi$.
	Then, $\theta \cup \{\varphi\}$ is a $\sem$-extension for each $\sem\in \{\adm,\comp,\stab\}$.

	(II.) $\theta\not\models \varphi$.
	Then, the set $C= \{C_i\mid \theta\not\models C_i\}$ of clauses is non-empty.
	Moreover, the set of arguments $\theta \cup C$ is conflict-free since no $C_i\in C$ is attacked by any literal $\ell\in \theta$ (i.e., there is no $\ell\in \theta \cap C_i$ for any $C_i\in C$).
	The set $\theta \cup C$ is also admissible as any $C_i\in C$ is defended against the attacker $\ell \in C_i$ by $\bar \ell \in \theta$ (recall that $\theta$ either contains $x$ or $\bar x$ for every variable $x$).
	The claims regarding stable and complete semantics use the same argument as in the proof of Claim~\ref{claim:clause-facets}.
	As a result, every literal $\ell\in \theta$ is contained in the $\sem$-extension $\theta\cup C$ and, is thus credulously accepted.

\begin{claim}\label{claim:phi-facets}
	The argument $\varphi\in A$ is a $\sem$-facet in $F_\varphi$ for each $\sigma\in\{\adm,\comp,\stab\}$iff the formula $\varphi$ is satisfiable.
\end{claim}
	\noindent\textbf{Proof of Claim.}
	If $\varphi$ is not satisfiable, then $\varphi$ does not belong to any $\sem$-extension in $F_\varphi$. As a result, $\varphi$ is not a facet.

	Suppose $\varphi$ is satisfiable and let $\theta$ be an assignment such that $\theta\models\varphi$.
	Then, $\theta\cup\{\varphi\}$ is a $\sem$-extension as established in Claim~\ref{claim:clause-facets} (C2).
	Now, due to Claim~\ref{claim:clause-facets} (C1), there is a $\sem$-extension covering each clause argument and hence not containing $\varphi$.
	As a result, $\varphi$ is a $\sem$-facet in $F_\varphi$.

	Combining Claims~(\ref{claim:clause-facets}--\ref{claim:phi-facets}), we achieve that (1) every argument $A\setminus\{\varphi\}$ is a $\sem$-facet and (2) the argument  $\varphi\in A$ is a $\sem$-facet iff $\varphi$ is satisfiable.
	As a result, $F_\varphi$ has exactly (1) $k$ $\sem$-facets iff $\varphi$ is satisfiable and (2) $k-1$ $\sem$-facets otherwise.

	Consequently the lemma follows since $|A|=k$.
\end{proof}

We continue with $\atleastkfacets_\sigma$ and prove that it is either $\NP$-complete or $\SigmaP$-complete depending on the choice of $\sigma$.

\begin{theorem}\label{thm:adm-atleastk}
	$\atleastkfacets_\sigma$ is $\NP$-complete for $\sigma\in \{\adm,\stab,\comp\}$.
\end{theorem}
\begin{proof}
	For membership, we guess $k$ distinct arguments $\{a_1,\ldots,a_k\}$, and simultaneously $2k$ $\sigma$-extensions $P_1,\dots,P_k, N_1,\dots,N_k$ such that: $a_i\in P_i$ and $a_i\not\in N_i$.
	The verification that each $S\in \{P_i,N_i \mid i\leq k\}$ is a $\sigma$-extension can be done in polynomial time.
	Then, $F$ has at least $k$ $\sigma$-facets iff each argument in $\{a_i\mid i\leq k\}$ is a $\sigma$-facet.
	
	For hardness, we utilize Lemma~\ref{lem:sat-kfacets}.
	Indeed, $\varphi$ is satisfiable iff the AF $F_\varphi$ has at least $k$ facets where $k= 2n+m+1$ for the formula $\varphi$ with $n$ variables and $m$ clauses.
\end{proof}

\begin{theorem}\label{thm:pref-atleastk}
	$\atleastkfacets_\sigma$ is $\SigmaP$-complete for $\sigma\in \{\pref,\semi,\stag\}$.
\end{theorem}
\begin{proof}
	For membership, we guess $k$ distinct arguments $\{a_1,\ldots,a_k\}$, and simultaneously $2k$ $\sigma$-extensions $P_1,\dots,P_k, N_1,\dots,N_k$ such that: $a_i\in P_i$ and $a_i\not\in N_i$.
	The verification that each $S\in \{P_i,N_i \mid i\leq k\}$ is a $\sigma$-extension can be done via an $\NP$-oracle.
	Then, $F$ has at least $k$ $\sigma$-facets iff each argument in $\{a_i\mid i\leq k\}$ is a $\sigma$-facet.
	This yields membership in $\NP^{\NP}$ (equivalently, $\SigmaP$).
	
	For hardness, we reduce from $\isfacet_\sem$ for $\sem \in \{\pref,\semi,\stag\}$.
	To this aim, let $F=(A,R)$ be an AF and $a\in A$ be an argument in the question.
	Assume that $|A|=n$.
	We let $n-1$ additional copies of $a$ and consider the set  $C_a = \{a_1,\dots,a_n\}$ of arguments where $a_1=a$ and $a_i\not\in A$ are fresh arguments for $i\geq 2$.
	Then, we construct the AF $F' = (A',R')$ where $A'= A\cup C_a$.
	The relation $R'$ consists of $R$ and additionally the following attacks:
	$\{(a_i,x) \mid (a,x)\in R, i\leq n\} \cup \{(x,a_i) \mid (x,a)\in R, i\leq n\}$.
	That is, $F'$ simply copies the argument $a$ together with all its incoming and outgoing attacks for each of the $n-1$ fresh arguments.
	
	We first prove that for a conflict-free (resp., admissible) set $S$ in $F$, adding arguments from $C_a$ to $S$ does not change its conflict-freeness (admissibility) in $F'$ as long as $S$ contains $a$.
        
	\begin{claim}\label{claim:conf}
		A set $S\subseteq A$ containing $a$ is conflict-free  (resp., admissible) in $F$ iff $S\cup C_a$ is conflict-free (admissible) in $F'$.
	\end{claim}
	\noindent\textbf{Proof of Claim.}
	We prove the case for conflict-freeness, the case for admissible semantics follows analogously.
	If $S$ is not conflict-free in $F$ then $S$ is also not conflict-free in $F'$ since $R\subseteq R'$.
	Conversely, suppose $S$ is conflict-free in $F$.
	Suppose to the contrary, there exists $x,y\in S\cup C_a$ such that $(x,y)\in R'$.
	Recall that $C_a$ is conflict-free in $F'$ by definition and $S$ is conflict-free in $F'\setminus C_a$.
	Then, it must be the case that $x\in S$ and $y\in C_a$ (or vise versa).
	But this leads to a contradiction to the conflict-freeness of $S$ since $a\in S$ and $(x,a_i)\in R'$ iff $(x,a)\in R$.
	Analogous case holds if $x\in C_a$ and $y\in S$. Thus $S\cup C_a$ is conflict-free in $F'$.

	We are now ready to prove the following claim.
	\begin{claim}\label{claim:copies}
		The argument $a$ is a $\sigma$-facet in $F$ iff each argument $a_i\in C_a$ is a $\sem$-facet in $F'$ for each $\sem \in \{\pref, \semi,\stag\}$.
	\end{claim}
	\noindent\textbf{Claim Proof.}
	Suppose $a$ is a $\sem$-facet in $F$.
	Then, there are $\sem$-extensions $S_1,S_2$ in $F$ such that $a\in S_1$ and $a\not \in S_2$.
	Then, we prove that each argument in $C_a$ belongs to some, but not all $\sem$-extensions of $F'$.
	Notice first that $S_1$ is not a $\sem$-extension in $F'$ for any $\sem\in\{\pref,\semi,\stag\}$.
	This holds due to the arguments in $C_a$.
	Indeed, if $S_1$ is a $\sem$-extension in $F$, then $S_1 \cup \{x \mid x\in C_a\}$ is a counter-example to $S_1$ being $\sem$-set in $F'$ due to Claim~\ref{claim:conf}. 
	
	\textbf{SomeE:} $S_1\cup C_a$ is a $\sem$-extension containing each $a_i\in C_a$ (again, due to Claim~\ref{claim:conf}).

	\textbf{NotAllE:} We prove that $S_2$ is a $\sem$-extension in $F'$ and $a_i\not \in S_2$ for each $a_i\in C_a$.
	We prove the claim for preferred semantics, other cases can be proven analogously.
	Since $S_2$ is a subset maximal admissible set in $F$ and $a\not\in S_2$, either $S_2 \cup \{a\}$ is not conflict-free, or not admissible in $F$.
	Consequently, either $S_2\cup \{a_i\}$ is not conflict-free, or not admissible (by Claim~\ref{claim:conf}).
	Since $A'\setminus A = C_a$, this proves that $S_2$ is a preferred set in $F'$.
	Similar arguments (with Claim~\ref{claim:conf}) yield results for the remaining two semantics.
	As a result, each $a_i \in C_a$ is a $\sem$-facet since $a_i\not\in S_2$ for each $a_i\in C_a$.

	Conversely, suppose $a$ is not a $\sigma$-facet in $F$.
	If $a$ is not in contained in any $\sem$-extension of $F$, then no argument from $C_a$ can be in any $\sem$-extension of $F'$ (using the same argument as in \textbf{SomeE}).
	Hence, no argument $a_i\in C_a$ is a $\sem$-facet in $F'$.
	Similarly, if $a$ belongs to every $\sem$-extension of $F$, once again we have that every $a_i\in C_a$ is contained in every $\sem$-extension in $F'$ (due to \textbf{NotAllE}).
	This results once again in no argument $a_i\in C_a$ being a $\sem$-facet in $F'$.

	We next observe that the argument $a$ is a $\sigma$-facet in $F$ iff the AF $F'$ admits at least $n$ $\sem$-facets for each $\sem \in \{\pref, \semi,\stag\}$.
	Indeed, $a$ is a $\sem$-facet in $F$ iff each $a_i\in C_a$ is also a facet in $F'$ due to Claim~\ref{claim:copies}. 
	Thus resulting in at least $n$ facets in $F'$.
	In contrast, if $a$ is not a facet in $F$ then no argument in  $a_i\in C_a$ is a facet.
	Hence $F'$ has at most $n-1$ facets.
\end{proof}

	It is worth remarking that the reduction from $\isfacet_\sem$ to $\atleastkfacets_\sem$ presented in the proof of Theorem~\ref{thm:pref-atleastk} does not work for admissible semantics.
	This holds since the converse direction of Claim~\ref{claim:copies} `if $a$ is not a facet in $F$ then no argument $a_i\in C_a$ is a facet in $F'$' is no longer true.
	Suppose $a$ is not a facet.
	Assume further that $a$ belongs to all admissible sets and there is at least one such set. 
	Now, take any admissible set $S$ in $F$.
	Clearly, $a\in S$, however, $S\setminus \{a\}\cup \{a_i\}$ is admissible in $F'$ and does not contain $a$.
	This results in every $a_i\in C_a$ being $\adm$-facet in $F'$.
	Consequently, $a$ is not an admissible-facet in $F$ although each of its copy in $F'$ is an admissible-facet.
	Moreover, $F'$ also admits at least $n-1$ $\adm$-facets in this case, thus violating the proof from Theorem~\ref{thm:pref-atleastk}.

 \begin{restatable}{theorem}{allatmostk}\label{thm:all-atmostk}
	$\atmostkfacets_\sigma$ is $\co\NP$-complete for $\sigma\in \{\adm,\stab,\comp\}$, whereas $\PiP$-complete for $\sigma\in \{\pref,$ $\semi,$ $\stag\}$.
\end{restatable}

\begin{proof}
	The complement $\atleastkfacets_\sigma$ of $\atmostkfacets_\sigma$ is $\NP$-complete (resp., $\SigmaP$-complete) for the mentioned semantics.
\end{proof}

\subsection{Exact $k$ Facets}
\label{sec:exact}

Perhaps unsurprisingly, $\exactkfacets$ also turns out to be easy for conflict-free and naive semantics.

 \begin{restatable}{theorem}{confexactk}\label{thm:conf-exactk}
   $\exactkfacets_\sigma$ is in $\Ptime$ for $\sigma\in \{\cnf,\nai\}$.
 \end{restatable}

 \begin{proof}
	Similar to the proof of Theorem~\ref{thm:conf-boundk}, one can count all the $\sem$-facets in polynomial time for $\sem\in\{\cnf,\nai\}$.
\end{proof}

For $\sigma\in \{\adm,\stab,\comp\}$ the problem $\exactkfacets_\sigma$ turns out to be more interesting since it is complete for the comparably esoteric class $\DP$.

\begin{theorem}\label{thm:adm-exactk}
	$\exactkfacets_\sigma$ is $\DP$-complete for $\sigma\in \{\adm,\stab,\comp\}$.
\end{theorem}

\begin{proof}
	The membership follows directly from $\atleastkfacets_\sem$~(Thm.~\ref{thm:adm-atleastk}) and $\atmostkfacets_\sem$~(Thm.~\ref{thm:all-atmostk}).

	For hardness, we reduce from SAT-UNSAT. 
	To this aim, we utilize Lemma~\ref{lem:sat-kfacets} for an instance $(\varphi,\psi)$ of SAT-UNSAT.
	We assume w.l.o.g. that $\varphi$ and $\psi$ do not share variables.
	Then, $\varphi$ is satisfiable and $\psi$ is not satisfiable iff the AF $F_\varphi$ has $k_1$ facets and $F_\psi$ has $k_2-1$ facets, where $F_i$ is the corresponding AF for $i\in \{\varphi,\psi\}$ with $k_i$ arguments.
	However, there is a small technical issue as the AF $F_\varphi\cup F_\psi$ can not distinguish the failure of the satisfaction of $\varphi$ from that of $\psi$.
	Therefore, we can not simply take the union $F_\varphi\cup F_\psi$ and let the number of facets be $k_1+k_2-1$.
	Nevertheless, we duplicate the argument $\varphi$ in $F_\varphi$ to yield the pair $\varphi, \varphi'$.
	Then $F_\varphi$ includes the additional attacks $(\varphi,\varphi'), (\varphi',\varphi)$ as well as $(C,\varphi')$ for each $C\in \varphi$.
	The resulting AF $F_\varphi$ has $k_1+1$ arguments and $\varphi$ is satisfiable iff $F_\varphi$ has $k_1+1$ facets.
	Note that the newly added argument $\varphi'$ is a $\sigma$-facet iff $\varphi$ is $\sigma$-facet for each $\sigma\in\{\adm,\comp,\stab\}$.
	Then, the theorem follows since $(\varphi,\psi)$ is a positive instance of SAT-UNSAT iff $\varphi$ is satisfiable and $\psi$ is not satisfiable iff $F_\varphi\cup F_\psi$ has exactly $(k_1+k_2)$ $\sigma$-facets for $\sigma\in\{\adm,\comp,\stab\}$.
\end{proof}

We conclude our complexity analysis by stating the  non-tight bounds for $\exactkfacets_\sem$ for the remaining semantics.

 \begin{restatable}{theorem}{prefexactk}\label{thm:pref-exactk}
	$\exactkfacets_\sigma$ is in $\DP_2$ for $\sigma\in \{\pref, \semi, \stag\}$.
\end{restatable}

\begin{proof}
	Follows directly from $\atleastkfacets_\sem$~(Thm.~\ref{thm:pref-atleastk}) and $\atmostkfacets_\sem$~(Thm.~\ref{thm:all-atmostk}).
\end{proof}

\section{Significance}\label{sec:significance}
Our notion of significance adopts a
decision-driven perspective. We define significance of arguments in terms of the influence of a decision to eliminate the
degree of freedom (on choices of remaining arguments). 
While counting approaches assess the plausibility of arguments in terms of their likelihood of being accepted, we measure how much the acceptance of an argument decreases freedom (or increases the significance of the decision). 
Intuitively, a higher significance score indicates that a specific decision does have a huge influence on the remaining facets. Furthermore, the number
of facets directly measures the \emph{amount of uncertainty} in extensions.
Consider an argument $a$ (as an \emph{opinion} or a \emph{view point}) and denote by $\bar a$ the complement/negation of $a$. 
E.g., an argument is approved ($a$) versus not approved ($\bar a$).
Then, a facet $\ell \in \{a,\bar{a}\}$ can be seen as the \emph{uncertainty} regarding~$a$, since $a$ can either be included in, or be excluded from certain extensions.
We next introduce the notion of \emph{approving} and \emph{disapproving} an argument.
Let $F=(A,R)$ be an AF, and $\sem$ be a semantics.
Recall that $\sem(F)$ denotes the collection of $\sem$-extensions in $F$.
Moreover, $\Cred\sem$ (resp., $\Skep\sem$) denotes the collection of credulously (skeptically) accepted arguments in $F$ under semantics $\sem$.
For an argument $a$, we let $\sem^a(F)$ denote the $\sem$-extensions in $F$ \emph{approving} the argument $a$.
Precisely, we define $\sem^a(F)=\{E\in \sem(F) \mid a\in E\}$.
Moreover, $\sem^{\bar a}(F)=\{E\in \sem(F)\mid a\not\in E\}$ represents the $\sem$-extensions in $F$ \emph{disapproving} $a$.
Now, let $\Cred\sem^a$ (resp., $\Skep\sem^a$) be the arguments in some (all) $E\in \sem^a(F)$.
Finally, $\Facet\sem^a(F)$ denotes the $\sem$-facets by considering only extensions in $\sem^a(F)$ (i.e., $\Cred\sem^a\setminus\Skep\sem^a$).

For an argument $a\in A$ and  $\ell\in\{a,\bar a\}$, we denote $\bar \ell = \bar a$ if $\ell=a$ and $\bar \ell =a$ for $\ell=\bar a$.
We say that $\ell$ is approved iff $\bar\ell$ is disapproved.
Approving a facet $\ell\in \{a,\bar a\}$ reduces the uncertainty regarding the remaining arguments in $A$ by restricting the extensions space to sets (not) containing $a$.
Further, approving $\ell$ can render a facet argument $b\in A$ non-facet.
This holds since, either (C1) $b\in E$ for each $E\in\sem^\ell(F)$ but $b\not\in E$ for each $E\in\sem(F)$, or (C2) $b\not\in E$ for any $E\in\sem^{\ell}(F)$ but $b\in E$ for some $E\in\sem(F)$.
Intuitively, we say that the uncertainty of such an argument $b$ has been \emph{resolved} by approving $\ell$.
Further, we say that approving $\ell$ results in the approval of $b$ in the case of (C1), and disapproval of $b$ if (C2) is the case.

Notice that approving (or disapproving) an argument results in fewer facets for every semantics $\sem$.
That is, the (dis)approval of any argument can not generate new facets.
Intuitively, we have less uncertainty than before after we (dis)approve certain arguments.
Precisely, we have the following lemma.

\begin{restatable}{lemma}{LessFacets}
\label{lem:facet-arguments}
	For any argument $a\in A$ and semantics $\sem$, $\Facet\sem^a(F)\subseteq \Facet\sem(F)$.
\end{restatable}

\begin{proof}
	Let $b\in A$ be such that $b\not\in \Facet\sem(F)$.
	If $b\in E$ for all $E\in\sem(F)$, then $b\in E$ for all $E\in\sem^a(F)$ as well, hence $b\not\in\Facet\sem^a(F)$.
	Conversely, if $b\not\in E$ for any $E\in\sem(F)$, then $b\not \in E$ for any $E\in\sem^a(F)$ as well, hence again $b\not\in\Facet\sem^a(F)$.
\end{proof}

Let $\sem$ be a semantics, $a\in A$ be a $\sem$-facet and $\ell\in\{a, \bar a\}$.
The observation that ``$\ell$ reduces the uncertainty among remaining arguments'' leads to the notion of \emph{significance} of $\ell$ 
under semantics $\sem$.
For an AF~$F$, we define:

\begin{equation}\label{eq:significance}
  \fimp_\sem[F,\ell] \eqdef
  \frac{\Card{\FA_\sem(F)} -  \Card{\FA_\sem^\ell(F)}}{\Card{\FA_\sem(F)}}.
\end{equation}

Intuitively, approving an argument $a$ is less significant if many uncertain arguments (facets) remain in $\Facet\sem^{a}(F)$.
Similarly, disapproving $a$ (and thus approving $\bar a$) is less significant if many facets remain in  $\Facet\sem^{\bar a}(F)$.

\begin{example}[Arguments Significance]\label{ex:sig}
	Reconsider the AF $F$ from Example~\ref{ex:limits} with stable extensions $\stab(F) = \{\{w, m, p\}, \{s, b, p\}, \{s, b, t\}\}$.
	While Example~\ref{ex:running} gave an intuition of significance,   Table~\ref{tab:ex tech} presents precise values for each argument.
	As outlined,  
	the argument $w$ has score $1$, and is thus more significant than $\bar w$ (score $2/3$).
	The argument $e$ not being a $\stab$-facet is excluded for significance reasoning.

	\begin{table}
		\centering
		\begin{tabular}{c@{}c@{\hspace{8pt}}c@{\hspace{8pt}}c@{\hspace{8pt}}c}
		  	\toprule
			$\ell \in$ & $\{w, m, t, \bar s, \bar b, \bar p\}$ & $\{s, b, \bar w, \bar m\}$ & $\{p, \bar t\}$ 
			\\
			\midrule
			$\Card{\FA_\stab^\ell(F)}$ 
			& $0$ & $2$ & $4$ 
			\\
			$\fimp_\stab[F,\ell]$ 
			& $1$ & $\frac{2}{3}$ & $\frac{1}{3}$  
			\\
			\bottomrule
		\end{tabular}
		\caption{Argument significance for the AF from Example~\ref{ex:sig}. 
		}
		\label{tab:ex tech}
	\end{table}
\end{example}

\subsection*{Computing Significance for Arguments}

Let $F=(A,R)$ be an AF, $\sem$ be a semantics and $a\in A$ be an argument.
Observe that the computation of $\fimp_\sem[F,\ell]$ (Equation~\ref{eq:significance}) requires counting facets in $\Facet\sem(F)$ and $\Facet\sem^\ell(F)$.
We argue that one can count the number of $\sem$-facets in a framework $F=(A,R)$ without having to enumerate or count all $\sem$-extensions explicitly.
In fact, $\Facet\sem(F)$ can be computed by asking $\isfacet_{\sigma}$ for each argument $a\in A$, which requires $|A|$-many queries.
Moreover, one can also count the exact facets in $\Facet\sem^\ell(F)$ without having to explicitly identify all $\sem^\ell$-extensions.
Observe that $\Facet\sem^\ell(F)$ corresponds to the result of remaining facets after approving $\ell$.

\begin{theorem}\label{thm:sig}
	Let $\sem$ be a semantics and $F=(A,R)$ be an AF.
	For every $a\in A$ and $\ell\in \{a,\bar a\}$ the sets $\Facet\sem(F)$ and $\Facet\sem^\ell(F)$ can be computed by a deterministic polynomial-time Turing machine with access to
	\begin{itemize}
        \item an $\NP$ oracle 
          for $\sem\in \{\adm,\comp,\stab\}$.
        \item 
        a $\SigmaPtime2$ oracle, 
          for $\sem\in \{\pref,\SemiSt,\stag\}$.
	\end{itemize}
\end{theorem}

\begin{proof}
	Given an AF $F=(A,R)$, semantics $\sem$ and argument $a\in A$.
	To compute $\Facet\sem(F)$, we consider the following procedure.
	For each $b\in A$:
	\begin{enumerate}
		\item Guess two sets $E_1,E_2 \subseteq A$,
		\item Check that $b\in E_1$, $b\not \in E_2$,
		\item Check that $E_i\in \sem(F)$,
		\item Answer ``Yes'' if each check is passed in Step $2-3$.
	\end{enumerate}
	This procedure is repeated polynomially-many times (precisely $|A|$-many times).
	As a post-processing, count the number of arguments $b\in A$, for which Step-4 answers ``Yes''.
	The 2nd step requires non-determinstic guesses, whereas the 3rd step needs (1) $\Ptime$ for $\sem\in \{\adm,\comp,\stab\}$ and (2) $\co\NP$, for $\sem\in \{\pref,\SemiSt,\stag\}$.
	Step-4 is again a final post-processing. 
	As a result, the procedure overall runs in the mentioned runtime for corresponding semantics.
	That is, $\Ptime$-time with an $\NP$ oracle for $\sem\in \{\adm,\comp,\stab\}$ and  ${\SigmaP}$-oracle, for $\sem\in \{\pref,\SemiSt,\stag\}$.
	
	To compute $\Facet\sem^\ell(F)$, we additionally include the following check to the Step-2 in the above procedure.
	\begin{enumerate}
		\item[2a.] Check that $a\in E_1$, $a\in E_2$ if $\ell = a$, and check $a\not\in E_1$, $a\not\in E_2$ if $\ell =\bar a$.
	\end{enumerate}
	This does not increase runtime, and completes the proof. 
\end{proof}
	Observe that one can not expect to lower the runtime in Theorem~\ref{thm:sig} (e.g., to $\NP$ in the case of $\sem\in \{\adm,\comp,\stab\}$).
	Intuitively, although each step (Step~$1-4$ in the proof of Theorem~\ref{thm:sig}) requires $\NP$-time, the final post-processing needs counting the number of arguments for which each check is passed.
	In fact, this would contradict Theorem~\ref{thm:adm-exactk} (unless $\NP=\DP$) since one can count all facets in an AF and determine whether this number equals $k$. 

\section{Implementation and Experiments}

\paragraph{Implementation}
We implemented counting of extensions and facets for various
semantics into our tool called \fargu (Facets for Reasoning and
Analyzing Meaningful Extensions).
We build on the \aspartix system, an ASP-based
argumentation system for Dung style abstract argumentation and
extensions thereof~\cite{EglyGagglWoltran08}.
We employ \aspartix's ASP encoding and take the ASP solver \clingo{}
version~5.7.1~\cite{GebserKaufmannSchaub12a} to compute credulous and
skeptical consequences and take set differences.
We leverage ASP as the solvers have native support for
enumerating consequences without exhaustive enumeration of all answer-sets~\cite{AlvianoDodaroFiorentino23a,GebserKaufmannSchaub09a}.

\paragraph{Design and Expectations}
We ran our experiments on a Ubuntu~11.4.0 Linux~5.15 computer with an
eight core Intel i7-14700 CPU 1.5~GHz machine with 64GB of RAM.
Each run is executed exclusively on the system.
To illustrate that we can count facets on practical instances and
obtain insights over counting, we take the admissible, stable
semantics, and semi-stable semantics over instances from the 3rd
International Competition on Computational Models of Argumentation
(ICCMA'19)~\cite{bistarelli2019third}.
This gives us 326 different argumentation frameworks of varying sizes
in the range of 0 to 10.000 arguments with an average of 800.3
arguments and a median of 160.0 arguments.
We take admissible, stable, and semi-stable semantics as
representative from each different level of hardness (see
Table~\ref{tab:results}).
We take the 2019 competition as instances are of reasonable size,
runtime, number for the scope of this experiment, and we can use input
instances without further modification.
We limit the runtime on each instance to 60 seconds for sustainability
reasons, as differences become visible already with the limitation,
and as a user might not want to wait long when investigating search
spaces. 
We collect the number of extensions and facets and measure the solver 
runtime. We have the following expectations:
  (i)~computing facets is faster than enumerating extensions,
  (ii)~facets are still accessible when the number of extensions is
  very high and enumeration takes longer runtime,
  (iii)~even when there are many extensions, there are
  reasonably small number of facets.

\newcommand{\mnull}[0]{\ensuremath{\cdot 10^0}}
\newcommand{\fff}[0]{\text{f}}
\begin{table}
  \centering
  \begin{tabular}{l@{~}r@{~~}c@{~~}Hr@{~~~}r@{~~}r@{~~~~}r@{~~~}}
    \toprule 
    $\sigma$      & n              & I              & s & \#$_e$                     & \#$_\fff$ & $t_{e}$ & $t_\fff$ \\
    \midrule                                                                             
    $\adm^\star$  & 128            & $[10^6,\cdot)$ & 0 & $^\dagger189.0 \cdot 10^6$ & $165.1$   & $61.0$  & $0.5$    \\
    $\adm$        & 11             & $[10^6,\cdot)$ & 1 & $27.1\cdot 10^6$           & $40.6$    & $6.2$   & $1.3$    \\
    $\adm$        & 187            & $(0,10^6)$     & 1 & $29.4 \cdot 10^3$          & $266.7$   & $0.8$   & $1.4$    \\
    $\stab^\star$ & 1              & $(0,10^6)$     & 0 & $^\dagger3.0 \mnull$       & ---       & $61.0$  & $61.0$   \\
    $\stab$       & $^\ddagger$310 & $(0,10^6)$     & 1 & $36.0 \mnull$              & $43.1$    & $0.7$   & $0.8$    \\
    $\stab$       & 14             & $[0,0]$        & 1 & $0.0 \mnull$               & $0.0$     & $1.4$   & $1.2$    \\
    $\semi^\star$ & 9              & $(0,10^6)$     & 0 & $^\dagger43.8 \mnull$      & $79.9$    & $61.0$  & $49.4$   \\
    $\semi$       & 177            & $(0,10^6)$     & 1 & $50.9 \mnull$              & $33.9$    & $3.7$   & $2.1$    \\
    $\semi^\star$ & 140            & $[0,0]$        & 0 & $^\dagger0.0 \mnull$       & ---       & $61.0$  & ---      \\
    \bottomrule
  \end{tabular}
  \caption{%
    Overview of results on enumerating extensions and
    computing facets for semantics~$\sigma\in\{\adm,\semi,\stab\}$ where
    $^\star$ marks timeouts (in order to distinguish the cases
    with/without timeout on enumerating extensions)
    and the columns contain
    the total number~$n$ of argumentation frameworks with
    the interval~$I=[a,b)$ referring to \mbox{$a \leq \#_e < b$};
    the average number~$\#_{e}$ and $\#_\fff$ of extensions/facets; 
    the average  runtime~$t_e$ and $t_\fff$ for enumerating extensions ($e$) / computing facets (f).
    The symbol $^\dagger$ illustrates that there is only a lower bound as computation did not finish.
    We excluded one instance ($^\ddagger$) due to timeout when computing credulous/skeptical extensions.
  }
  \label{tab:exps}
\end{table}

\paragraph{Observations and Summary} Table~\ref{tab:exps} presents a survey of our results. Details on the evaluation are available in the
supplemental data.
We see that the number of admissible extensions can be larger than
$10^6$ and clingo fails to enumerate all extensions ($^\star$) if the
number of extensions is very high.
For the admissible extensions, we observe that even if the number of
extensions is quite high, the number of facets remains reasonably
small making it interesting to diversify extensions, or investigate
more details about those arguments, which still allow flexibility.
The observations confirm our expectations for the admissible
semantics.
However, we gain only limited insights for the semantics~$\stab$ and
$\semi$.
Here the number of extensions is fairly low and the solver either
manages to enumerate all extensions, or already fails to solve one.

\section{Conclusion}
We defined a new perspective on exploring significance of arguments in
extensions of an abstract argumentation framework.
We present a comprehensive complexity analysis for deciding whether an
argument is a facet ($\isfacet$) and deciding whether an argumentation
framework has at least~$k$ facets ($\atleastkfacets_\sigma$), at
most~$k$ facets ($\atmostkfacets_\sigma$), and exactly~$k$ facets
($\exactkfacets_\sigma$).
We establish that the complexity ranges between $\Ptime$ and $\DP_2$,
including tight lower bounds for most cases (see
Table~\ref{tab:results}).
While our primary focus lies on establishing a comprehensive
complexity picture, our implementation allows computing the number of facets practically
for concrete abstract argumentation frameworks building on top of
existing solvers.

For future work, we plan to investigate techniques whether
significance originating from facets can be extended to arguments
depending on each other and notions of fairness in argumentation
frameworks.
We also believe that the missing case for
$\sigma\in \{\pref, \semi,\stag\}$ remains interesting to study. We
expect that the problem is also hard, as the decision problems for at
least and at most are ${\SigmaP}$- and $\PiP$-complete.
From a practical perspective, we believe that it would be interesting
to integrate facet-based reasoning and significance computation into modern SAT-based argumentation solvers.
Moreover, investigating facets for other formalisms such abductive
reasoning~~\cite{MahmoodMeierSchmidt20} or default
logic~\cite{FichteHecherSchindler22} seems interesting as well as
closing the gap to the topic of
inconsistencies~\cite{FichteHecherMeier21,FichteHecherSzeider23}.

\section*{Acknowledgments}
The work has received funding from 
the Austrian Science Fund (FWF), grants J 4656 and P 32830, 
the Deutsche Forschungsgemeinschaft (DFG, German Research Foundation), grants TRR 318/1 2021 – 438445824 and ME 4279/3-1 (511769688), the European Union's Horizon Europe research and innovation programme within project ENEXA (101070305), 
the Society for Research Funding in Lower Austria (GFF, Gesellschaft für Forschungsf\"orderung N\"O), grant ExzF-0004, 
the Swedish research council under grant VR-2022-03214,
as well as the Vienna Science and Technology Fund (WWTF), grant ICT19-065, and by the Ministry of Culture and Science of North Rhine-Westphalia (MKW NRW) within project WHALE (LFN 1-04) funded under the Lamarr Fellow Network programme,
and ELLIIT funded by the Swedish government.

Johannes and Yasir appreciate valuable discussions with Lydia Blümel and Matthias Thimm  (Fern Universität in Hagen) on an alternative proof to Theorem 21.

\bibliographystyle{abbrv}
\bibliography{argu_facets_ijcai25}
\end{document}